\def\year{2020}\relax
\documentclass[letterpaper]{article} %

\newcommand{\lpalg}{\ensuremath{\mathsf{NAdap}}\xspace}

\newcommand{\pro}{\ensuremath{\mathsf{Profit}}\xspace}

\def \OPTP {\ensuremath{\operatorname{OPT-P}}\xspace}
\def \OPTF {\ensuremath{\operatorname{OPT-F}}\xspace}

\usepackage[autostyle]{csquotes} 
\usepackage[sumlimits]{amsmath}
\usepackage{mathtools}
\usepackage{xargs}
\usepackage{amsthm, amssymb}
\usepackage{latexsym}
\usepackage{fullpage}
\usepackage{newtxtext,newtxmath}              %
\usepackage[T1]{fontenc}
\usepackage{color}
\usepackage{comment}
\usepackage{graphicx}
\usepackage{epstopdf}
\usepackage[linesnumbered, ruled, vlined, norelsize]{algorithm2e}
\usepackage{parskip}    %
\usepackage{sidecap}
\usepackage{floatrow}

\newtheorem{theorem}{Theorem}
\newtheorem{lemma}{Lemma}

\usepackage{algorithmic}

\newcommand{\E}{\mathbb{E}}

\usepackage{mathrsfs}

\newcommand{\Del}{\Delta}
\newcommand{\ep}{\epsilon}

\newcommand{\alp}{\alpha}

\def \OPT {\ensuremath{\operatorname{OPT}}\xspace}

\def \ALG {\ensuremath{\operatorname{ALG}}\xspace}
\newcommand{\LP}{\ensuremath{\operatorname{LP}}\xspace}
\newcommand{\NALG}{\ensuremath{\operatorname{NADAP}}\xspace}

\def \gre {\ensuremath{\operatorname{Greedy}}\xspace}
\def \uni {\ensuremath{\operatorname{Uniform}}\xspace}

\newcommand{\x}{\vec{x}}

\newcommand{\z}{\vec{z}}

\newcommand{\y}{\vec{y}}

\newcommand{\SF}{\mathsf{SF}}

\makeatletter
\@ifundefined{theorem}{%
\newtheorem{theorem}{Theorem}
}{}
\@ifundefined{example}{%
\newtheorem{example}{Example}
}{}

\@ifundefined{conjecture}{%

}{}
\@ifundefined{lemma}{%
\newtheorem{lemma}{Lemma}
}{}
\@ifundefined{claim}{%

}{}
\@ifundefined{observation}{%

}{}
\@ifundefined{definition}{%

}{}
\@ifundefined{question}{%

}{}
\@ifundefined{corollary}{%

}{}
\@ifundefined{fact}{%

}{}

\makeatother

\newcommand{\xhdr}[1]{ \noindent{\textbf{#1}}}

\newcommand{\cI}{\mathcal{I}}

\newcommand{\cM}{\mathcal{M}}

\newcommand{\ie}{\emph{i.e.,}\xspace}
\newcommand{\eg}{\emph{e.g.,}\xspace}

\usepackage[multiple]{footmisc}
\usepackage[group-separator={,}]{siunitx}

\ifodd 0

\else

\fi

\usepackage{aaai20}  %
\usepackage{times}  %
\usepackage{helvet}  %
\usepackage{courier}  %
\usepackage{url}  %
\usepackage{graphicx}  %
\usepackage{subfigure}
\usepackage{xcolor}
\usepackage{xspace}
\newcounter{int}
\makeatletter
\newcommand{\citen}[1] {\setcounter{int}{0}\@for\tmp:=#1\do{%
\ifnum \value{int}>0; \fi%
\setcounter{int}{1}%
\citeauthor{\tmp} \shortcite{\tmp}}}
\makeatother
\makeatletter
\newcommand{\citenp}[1]{\setcounter{int}{0}\@for\tmp:=#1\do{%
\ifnum \value{int}>0; \fi%
\setcounter{int}{1}%
\citeauthor{\tmp}, \citeyear{\tmp}}}
\makeatother

\begin{document}

\title{Balancing the Tradeoff between Profit and Fairness in Rideshare Platforms During High-Demand Hours}
\author{Vedant Nanda,\textsuperscript{1,2}
Pan Xu,\textsuperscript{3}
Karthik Abinav Sankararaman,\textsuperscript{1,4}
John P. Dickerson,\textsuperscript{1}
Aravind Srinivasan\textsuperscript{1}\\
\textsuperscript{1}{University of Maryland, College Park}\\
\textsuperscript{2}{Max Planck Institute for Software Systems}\\
\textsuperscript{3}{New Jersey Institute of Technology, New Jersey}\\
\textsuperscript{4}{Facebook}\\
vedant@cs.umd.edu,
pxu@njit.edu,
karthikabinavs@gmail.com,
john@cs.umd.edu,
srin@cs.umd.edu 
}

\maketitle

\begin{abstract}
Rideshare platforms, when assigning requests to drivers, tend to maximize profit for the system and/or minimize waiting time for riders. Such platforms can exacerbate biases that drivers may have over certain types of requests. We consider the case of peak hours when the demand for rides is more than the supply of drivers. Drivers are well aware of their advantage during the peak hours and can choose to be selective about which rides to accept. Moreover, if in such a scenario, the assignment of requests to drivers (by the platform) is made only to maximize profit and/or minimize wait time for riders, requests of a certain type (\eg from a non-popular pickup location, or to a non-popular drop-off location) might never be assigned to a driver. Such a system can be highly unfair to riders. However, increasing fairness might come at a cost of the overall profit made by the rideshare platform. To balance these conflicting goals, we present a flexible, non-adaptive algorithm, \lpalg, that allows the platform designer to control the profit and fairness of the system via parameters $\alpha$ and $\beta$ respectively. We model the matching problem as an online bipartite matching where the set of drivers is offline and requests arrive online. Upon the arrival of a request, we use \lpalg to assign it to a driver (the driver might then choose to accept or reject it) or reject the request. We formalize the measures of profit and fairness in our setting and show that by using \lpalg, the competitive ratios for profit and fairness measures would be no worse than $\alpha/e$ and $\beta/e$ respectively. Extensive experimental results on both real-world and synthetic datasets confirm the validity of our theoretical lower bounds. Additionally, they show that $\lpalg$ under some choice of $(\alpha, \beta)$ can beat two natural heuristics, Greedy and Uniform, on \emph{both} fairness and profit. Code is available at: \url{https://github.com/nvedant07/rideshare-fairness-peak/}.
\end{abstract}

\section{Introduction}\label{sec:intro}

 Rideshare platforms have received significant attention in both the computer science and operations research communities. The following are the three main categories of research. The first studies matching policy design in the rideshare setting  (\ie matching riders and drivers), \eg \cite{xu-aaai-19,ashlagi2019edge,Patrick-18-JAI,tong2016icde,tong2016vldb,tong2017flexible,BeiZ18,dickerson2018assigning,xuAAAI18}. The second considers the spatial-temporal pricing aspects of rideshare, \eg \cite{ma2018spatio,bimpikis2019spatial,kanoria2019near,Banerjee-ec-17,Banerjee:2016}. The third focuses on applying reinforcement-learning approaches to planning and matching problems in rideshare see, \eg \cite{xu2018large,lin2018efficient}.

	In all the aforementioned prior work, the objective is either to maximize the total profit of the system or minimize the waiting time of riders (or a combination thereof). However, both of these objectives are \emph{global} in that they do not ensure sub-group level fairness (\eg riders of a particular protected class being systematically underserved). Consider a scenario during peak hours when there is increased demand for rides, thus, giving the drivers a bargaining advantage to drop rides. Typical criteria used by drivers to reject riders include riders' starting/ending location, trip length, gender, race and age. Recently, it has been reported that drivers also reject riders based on attributes such as gender, race and disability either intentionally or unintentionally, see, \eg \cite{web-female,web-disability,web-guided-dog}. \citen{web-race} reported that ``black passengers using ride-hailing apps have to wait an average of $1$ minute and $43$ seconds longer than white counterparts and are 4\% more likely to have drivers cancel on them.''

	Thus, current rideshare platforms can enable and amplify prejudices in society. To counter this, current rideshare apps implement the following measures:\footnote{See, for example, \S4A ``Sharing Between Users'' in Lyft's privacy policy found at \url{https://www.lyft.com/privacy}.} (1) Riders' photo and destination are hidden from the driver until they accept/reject the request; (2) Penalty is imposed if drivers cancel a certain  number of trips after initially accepting them (\eg drivers' account getting temporarily deactivated).  However, drivers devise new strategies to avoid the restrictions imposed by the above measures. It was reported~\cite{web-cancel} that some Uber drivers bypass the limitation imposed by (1) by starting the trip moments before picking up a passenger to see where the passenger is going. In some other cases, drivers get around the limitation imposed by (2) by intentionally delaying the pickup, thus forcing the rider to cancel the trip instead.

	Ensuring group level fairness and optimizing global profit tends to be somewhat conflicting goals in general; particularly so in rideshare platforms during peak hours. The rationale is that the driver to request ratio is small and thus, drivers can afford to be more choosy. This may lead to unfair practices such as rejecting trips to unpopular destinations and rejecting riders with disabilities. To promote group level fairness among the riders, systems should aggressively match requests which typically incur high cancellation rates. This action, however, would increase the number of cancellations by drivers. The current penalty in popular rideshare platforms is that drivers are deactivated if they cancel too many requests. Thus, prioritizing high risk (in terms of cancellation) trips could lead to more drivers being deactivated and/or leaving the system. In summary, on one hand, prioritizing group level fairness leads to limited drivers and thus, limited total trips. On the other hand, letting the market find its equilibrium leads to amplifying societal biases. Thus, the central question is the following: \emph{Can we design policies that can smoothly tradeoff between the two conflicting objectives}? 

	In this paper, we answer the above question in the affirmative by providing \emph{provably} efficient policies. As is common (see, e.g., \citen{xuAAAI18} for additional motivation) we model the dynamics of the rideshare platform as an online-matching model as follows. We have a bipartite graph $G=(U,V,E)$, where $U$ and $V$ represent the set of available drivers (static or offline) and the request types (dynamic or online arrival) respectively.\footnote{We consider a small time window during the peak hours, thus assuming that the set of drivers are static.} We use the notation $m:= |U|$ and $n :=|V|$ throughout this paper. Each driver \emph{type} represents a specific group (\eg gender, age and race) in a given location, while each request type represents a specific group with a given starting and ending location. There is an edge $f=(u,v)$ if  $u$ is capable of serving the request (of type) $v$ (\ie the distance between them is below a given threshold). The online phase consists of $T$ time-steps with the sets $U$ and $V$ known to the algorithm. In each time-step, a request $v \in V$ arrives and is presented to the algorithm. Upon its arrival an \textit{immediate and irrevocable} decision is required: either reject $v$, or match $v$ with an available driver in $U$. WLOG we assume that each $u$ has a unit capacity (which can be matched only once).\footnote{We can create multiple copies of $u$ to address the case when each $u$ can be matched multiple times. Note that each copy corresponds to the same driver and shares same underlying variables.}

	We have the following key assumptions that we use to show provably guarantees. In the experimental section, we work with real data and show that the algorithms are robust even when some of these assumptions do not necessarily hold.

	\xhdr{Arrival of Requests}. We consider a finite set of $T$ requests $v \in V$ that are drawn from a \emph{known} identical distribution independently; this is commonly~\cite{xuAAAI18} called the \textit{known identical independent distributions} (KIID). The motivation for this assumption stems from the fact that we can often learn the arrival distribution from historical logs~\cite{Yao2018deep,DBLP:conf/kdd/LiFWSYL18,DBLP:conf/kdd/WangFY18}. KIID is widely used in many practical applications of online matching markets including rideshare and crowdsourcing~\cite{xu-aaai-19,dickerson2018assigning,singer2013pricing,singla2013truthful}. 
	Further, we call the expected number of times any request $v$ is sampled from this distribution in the $T$ rounds as the \emph{arrival rate}, denoted by $r_v$. Thus, it is easy to see that $\sum_{v \in V} r_v = T$. We further assume that the total number of arrivals of online requests is far larger than that of drivers in the system, \ie $T \gg |U|$.

\xhdr{Edge existence probabilities}. Each edge $f=(u,v)$ is associated with an existence probability $p_f \in (0,1]$:  once we assign $v$ to $u$, we observe an immediate random outcome of the existence, which is present (\ie $u$ accepts $v$) with probability $p_f$ and not ($u$ cancels $v$) otherwise.
The probability $p_f$ captures the statistical chance that a driver of type $u$ would accept to serve the request of type $v$. We assume that (1) the randomness driving the edge existence is independent across all the edges; (2) the values $p_f$ are provided as part of the input. The first assumption is motivated by individual choice and the second from the fact that historical logs can be used to compute such statistics with high precision.

\xhdr{Cancellation budget}. Each driver $u$ is associated with a given budget of cancellation, $\Del_u \in \mathbb{Z}^{+}$. In other words, driver $u \in U$ will be removed from the graph $G$ if they cancel more than $\Del_u$ requests in the $T$ rounds (in which case we assume that $u$ is temporarily deactivated by the system as a penalty). Once a request $v$ gets rejected, we assume that the system will not try any reassignment of $v$ to other available drivers. This is without loss of generality since any reassignment can be modeled as resampling in the succeeding time-steps.

We assume that the system gains a profit $w_f$ from $f=(u,v)$ if driver $u$ completes (\ie is assigned and the driver accepts) the trip $v$ (in this case, we call it a \emph{successful} assignment or a match). For a given policy $\ALG$, let $\cM$ (possibly random) be the set of successful assignments; we interchangeably use the term \emph{matching} to denote this set $\cM$. We define two objectives, namely \emph{profit} and \emph{fairness}, as follows. 

\begin{description}
\item [\textbf{Profit}:] The expected total profit over all matches, which is defined as $ \E[\sum_{e \in \cM} w_e ]$.

\item [\textbf{Fairness}:] Let $\cM_v \subseteq \cM$ be the subset of edges incident to $v$. Note that $|\cM_v|$ can be larger than $1$ due to multiple arrivals of type $v$ in the $T$ time-steps.
	We define the fairness achieved by \ALG over all request types as  $\min_{v \in V} \frac{\E[|\cM_v|]}{r_v}$, which refers to the minimum ratio of the expected number of matches of type $v$ to that of arrivals.\footnote{Since we are in the context of peak hours, we consider the group level fairness of riders.} Thus, maximizing fairness corresponds to maximizing this minimum ratio. 
\end{description} 

We aim to design an online matching policy that balances the tradeoff between the two objectives of maximizing profit and fairness. 

\subsection{Preliminaries and Main Contributions}
\xhdr{Competitive ratio.} 
Competitive ratio is a commonly-used metric to evaluate the performance of online algorithms. 
Consider an online maximization problem for example. 
Let $\ALG(\mathcal{I})=\E_{I \sim \cI} [\ALG(I)]$ denote the expected performance of $\ALG$ on a given distribution $\mathcal{I}$, where the expectation is taken over the random arrival sequence $I$.  
Let $\OPT(\mathcal{I})=\E[\OPT(I)]$ denote the expected \emph{offline optimal}, where $\OPT(I)$ refers to the optimal value after we observe the full arrival sequence $I$. 
Then, the competitive ratio is defined as $\min_{\mathcal{I}} \frac{\ALG(\mathcal{I})}{\OPT(\mathcal{I})}$. 
It is a common technique to use a Linear Program (\LP) to upper bound $\OPT(\mathcal{I})$ (called the benchmark \LP) and hence get a valid lower bound on the target competitive ratio.  
In our paper, we conduct competitive-ratio analysis on both objectives.

\xhdr{Main contributions.} This paper provides three-fold contributions.

First, we formalize the metric of fairness in rideshare. More specifically, we consider the online-matching based model with multiple objectives.

Second, we present a provably efficient algorithm and provide formal mathematical guarantees. To do so,  we first propose a bi-objective linear program (\LP-\eqref{obj-1} and \LP-\eqref{obj-2}), whose optimal value is at least as large as that of any online algorithm that maximizes either of these objectives (or a combination thereof). Our main algorithm \lpalg uses this bi-objective LP to guide the online decision-making process. In particular, we prove the following main theorems.

\begin{theorem}\label{thm:main-1}
$\lpalg(\alp, \beta)$ achieves a competitive ratio at least  $\Big(\alp /e, \beta /e\Big)$ simultaneously on the profit and fairness for any given $\alp, \beta>0$ with $\alp+\beta \le 1$. 
\end{theorem}

\begin{theorem}\label{thm:hard}
No non-adaptive algorithm can achieve a $(\alp, \beta)$-competitive ratio simultaneously on the profit and fairness with $\alp+\beta>1-1/e$ using \LP-\eqref{obj-1} and \LP-\eqref{obj-2} as the benchmark.
\end{theorem}

 Third, we provide an extensive evaluation of the algorithm and modeling assumption on a real-world dataset collected from a large on-demand taxi dispatching platform. The experiments have many novel insights; among others, we show that even when some of the assumptions that were used to prove mathematical guarantees do not hold, \lpalg performs well in practice.
 
\section{Additional Related Work}

There is a large body of work which studies fairness issues in resource allocation problems (divisible or indivisible goods), see, \eg\cite{ghodsi2011dominant,bateni2018fair,parkes2015beyond,kash2014no,fain2018fair}. Most of these works require the allocation policy to satisfy certain properties in fair mechanism design such as strategy-proofness and envy-freeness. These properties do not apply here, however. Recent work by \citen{suhr2019} proposes a matching policy to balance fairness and profit over time; however, they do not provide any optimality guarantess for their proposed policy. \citen{lesmana2019neurips} has a setting similar to ours, however they consider the case when there are more drivers than riders. Our focus is on the case when the number of drivers in the system are less than the demand for rides (\ie peak hours).
  
Our model belongs to a more general optimization paradigm, called \emph{Multi-Objective Optimization}. \citen{ravi1993many} presented approximation algorithms for a variety of network-design problems. \citen{grandoni2009} designed several iterative-rounding based approximation algorithms for multi-objective optimization problems. More recently, \citen{aggarwal2014} studies the Bi-objective Online Bipartite Matching where there is essentially one single objective: the minimum matching ratio over two disjoint sets of edges. \citen{esfandiari2016bi} considers the Bi-objective Online Submodular Optimization problem, where the two objectives are two monotone submodular functions. Our objectives are two linear functions, but our model assumes a more complicated setting (\ie edge existence probabilities and cancellation quotas on the offline-side vertices).  The models studied by \citen{bansal2012lp} and~\citen{BSSX17} have the closest setting to us: each edge has an independent existence probability and each vertex from the offline and/or online side has a patience constraint (similar to cancellation quota for each driver in our setting) on it. However, both those works investigated only one single objective: maximization of the total profit.

\section{Valid Benchmarks for Profit and Fairness}

We first present our benchmark LPs and then an LP-based parameterized algorithm.
For each edge $f=(u,v)$, let $x_f$ be the expected number of probes on edge $f$ (\ie assignments of $v$ to $u$ but not necessarily matches) in the offline optimal. For each $u$ ($v$), let $E_u$ ($E_v$) be the set of neighboring edges incident to $u$ ($v$). Consider the following bi-objective LP.

\begin{alignat}{2}
\max & ~~\sum_f w_f x_f p_f  && \label{obj-1} \\
\max \min_{v \in V} & ~~\frac{\sum_{f \in E_v} x_f p_f}{r_v} &&  \label{obj-2} \\
\text{s.t.} & \sum_{f \in E_u} x_f p_f \le 1  &&~~ \forall u \in U \label{cons:match-u} \\ 
 & \sum_{f \in E_u} x_f \le \Delta_u  && ~~ \forall u \in U \label{cons:pat-u} \\ 
  & \sum_{f \in E_v} x_f \le r_v  && ~~\forall v \in V \label{cons:pat-v}  \\
  & 0 \le x_f  && ~~ \forall f \in E\label{cons:edge}
\end{alignat}

Let \LP-\eqref{obj-1} and \LP-\eqref{obj-2} denote the two LPs with the respective Objective \eqref{obj-1} and   \eqref{obj-2}, each with constraints \eqref{cons:match-u}, \eqref{cons:pat-u}, \eqref{cons:pat-v}, \eqref{cons:edge}. Note that we can rewrite Objective \eqref{obj-2} as a linear one like $\max \eta$ with additional linear constraints as $\eta \le \frac{\sum_{f \in E_v} x_f p_f}{r_v} $ for all $v \in V$. For presentation convenience, we keep the current compact version. The validity of \LP-\eqref{obj-1} and \LP-\eqref{obj-2} as benchmarks for our two objectives can be seen in the following lemma.

\begin{lemma} \label{lem:LP}
 \LP-\eqref{obj-1} and \LP-\eqref{obj-2} are valid benchmarks for the two respective objectives, profit and fairness. In other words, the optimal values to \LP-\eqref{obj-1} and \LP-\eqref{obj-2} are valid upper bounds for the expected profit and fairness achieved by the offline optimal respectively.
\end{lemma}

\begin{proof}
We can verify that objective functions~\eqref{obj-1} and~\eqref{obj-2} each captures the exact expected profit and fairness achieved by the offline optimal, according to our definition in Section~\ref{sec:intro}. To prove the validity of the benchmark for each objective, it suffices to show the feasibility of all constraints for any given offline optimal. 

Constraint \eqref{cons:match-u} is valid since each driver $u$ has a unit capacity; Constraint \eqref{cons:pat-u} is valid since each $u$ can be probed at most $\Delta_u$ times according to our assumption; Constraint \eqref{cons:pat-v} is valid since the expected number of probes related to each $v$ should be no more than that of online arrivals (recall that we can try at most one assignment upon the arrival of $v$). Thus we justify the feasibility of all constraints for any given offline optimal. 
\end{proof}

\section{An LP-based Parameterized Algorithm}
Now we present an LP-based parameterized algorithm. Let $\{x^{*}_{f}\}$ and $\{y^*_{f}\}$ be an optimal solution to \LP-\eqref{obj-1} and \LP-\eqref{obj-2} respectively. Consider a given pair of parameters $(\alp, \beta)$ with $0\le \alp, \beta \le 1, \alp+\beta \le 1$. $\lpalg(\alp,\beta)$ will sample an assignment guided by the two LP solutions $\{x^{*}_{f}\}$ and $\{y^*_{f}\}$,  with probabilities $p$ and $q$ respectively. Note that constraint \eqref{cons:pat-u} for the LP models the expectation of number of probes. While this number might exceed $\Delta_u$, it's important to note that this solution is for the offline case and does not correspond to the actual assignments. During the execution of our algorithm, we explicitly check if a driver is available before making the assignment, thus ensuring that no driver is assigned more than $\Delta_u$ rides. The details are as follows.

\begin{algorithm}[h!]
\DontPrintSemicolon
Let $v$ arrive at time $t$. \;
With probability $\alp$, sample an edge (assignment) $f=(u,v) \in E_v$with probability $x^{*}_{f}/r_v$. Assign $v$ to $u$ iff $u$ is available at $t$. \;
 With probability $\beta$, sample an edge (assignment) $f'=(u',v) \in E_v$with probability $y^*_{f}/r_v$. Assign $v$ to $u'$ iff $u'$ is available at $t$. \;
 With probability $1-\alp-\beta$, reject $v$.

\caption{An LP-based non-adaptive algorithm: $\lpalg(\alp, \beta)$}
\label{alg:lp-alg}
\end{algorithm}

The key part to prove Theorem \ref{thm:main-1} is the computation of the probability that each driver is available at each time. Focus  on a given driver $u$ and a time $t \in [T]$. Let $\SF_{u,t}$ be the event that $u$ is \emph{available} at (the beginning of) $t$. Note that the occurrence of $\SF_{u,t}$ can be guaranteed by these two events: (1)    $u$ never \emph{received and simultaneously accepted} any assignment prior to $t$; (2)  $u$ received no more than $\Del_u-1$ assignments prior to $t$.  For each $f=(u,v)  \in E_u$ and $\ell<t$, let $X_{f,\ell}$ indicate if $f$ comes (or $v$ comes) at time $\ell$, $Y_{f,t}$ indicate if $f$ gets sampled in $\lpalg(p,q)$ at $t$ and $Z_{f,t}$ indicate if $u$ accepts $f$ after the assignment. Set $A_{u,t}\doteq \sum_{\ell<t} \sum_{f \in E_u} X_{f,\ell} Y_{f,\ell}Z_{f,\ell} $ and $B_{u,t}\doteq \sum_{\ell<t} \sum_{f \in E_u} X_{f,\ell} Y_{f,\ell}$. From our analysis, we see that $\Pr[\SF_{u,t}] \ge \Pr[(A_{u,t}=0) \wedge (B_{u,t} \le \Del_u-1)]$.

\begin{lemma}\label{lem:1}
For any given $u$ and $t \in [T]$, we have
$$\Pr[A_{u,t}=0] \ge \Big(1-\frac{1}{T}\Big)^{t-1}, ~ \Pr[B_{u,t}\le \Del_u-1] \ge 1-\frac{t-1}{T}$$
\end{lemma}

\begin{proof} \begin{align*}
\Pr[A_{u,t}=0] &=\prod_{\ell<t} \Pr\Big[\sum_{f \in E_u} X_{f,\ell} Y_{f,\ell} Z_{f,\ell}=0\Big]\\
 &=\prod_{\ell<t} \left( 1-\Pr\Big[ \sum_{f \in E_u}X_{f,\ell} Y_{f,\ell} Z_{f,\ell} \ge 1 \Big]\right) \\
 &=\prod_{\ell<t} \left(  1-\sum_{f \in E_u} \frac{r_v}{T} \Big(\alp \frac{x_f^*}{r_v}+\beta \frac{y_f^*}{r_v} \Big) p_f \right)\\
 & =\prod_{\ell<t} \left(  1-\frac{1}{T}\sum_{f \in E_u}  \Big(\alp x_f^* p_f+\beta y_f^* p_f \Big)\right)  \\
 &\ge \Big(1-\frac{1}{T}\Big)^{t-1}
 \end{align*}
The last inequality follows from $\sum_{f \in E_u} x_f^* p_f \le 1$ and $\sum_{f \in E_u} y_f^* p_f \le 1$ due to Constraint \eqref{cons:match-u}, and $\alp+\beta \le 1$.

As for the second inequality, we have
\begin{align}
\Pr[B_{u,t} \le \Del_u-1] =1-\Pr[B_{u,t} \ge \Del_u] \ge 1-\frac{1}{\Del_u} \E[B_{u,t}] \label{ineq:pat-1}
\end{align}
Note that 
\begin{align}
\E[B_{u,t}] &=\sum_{\ell<t} \sum_{f \in E_u} \E[X_{f,\ell} Y_{f,\ell}] 
=\sum_{\ell<t} \sum_{f \in E_u} \frac{r_v}{T}  \Big( \frac{\alp x_f^*}{r_v}+ \frac{\beta y_f^*}{r_v} \Big) \\
& \le \frac{1}{T}\sum_{\ell<t} (\alp+\beta) \Delta_u \le \frac{t-1}{T} \Del_u  \label{ineq:pat-2}
\end{align}
The last inequalities above follows from $\sum_{f \in E_u} x_f^* \le \Del_u$ and $\sum_{f \in E_u} y_f \le \Del_u$ due to Constraint \eqref{cons:pat-u}, and $\alp+\beta \le 1$. Substituting Inequality \eqref{ineq:pat-2} to \eqref{ineq:pat-1}, we have that $\Pr[B_{u,t} \le \Del_u-1] \ge 1-(t-1)/T$. 
\end{proof}

\begin{lemma}\label{lem:2}
$$\Pr[B_{u,t} \le \Del_u-1 | A_{u,t}=0 ] \ge 1-\frac{t-1}{T}$$
\end{lemma}
Lemma~\ref{lem:2} says that the two events $(A_{u,t}=0)$ and $(B_{u,t} \le \Del_u-1)$ are positively correlated. Proof of Lemma~\ref{lem:2} can be found in the full version of the paper (on Arxiv).

Now we have all ingredients to prove the main Theorem~\ref{thm:main-1}. 
\begin{proof}
From Lemmas~\ref{lem:1} and \ref{lem:2}, we have for each $u$ and $t \in [T]$,
\begin{align*}
& \Pr[\SF_{u,t}]  \ge \Pr[(A_{u,t}=0) \wedge (B_{u,t} \le \Del_u-1)] \\
 & \ge \Pr[A_{u,t}=0] \Pr[B_{u,t} \le \Del_u-1] \ge \Big(1-\frac{1}{T}\Big)^{t-1} \Big(1-\frac{t-1}{T}\Big)
 \end{align*}  
 
 For each $f \in E$, let $\kappa_f$ be the expected number of \emph{successful} assignments of $f$ in $\lpalg(\alp, \beta)$. Here an assignment $f=(v,u)$ is successful iff $u$ is available when we assign $v$ to $u$ (but no necessarily means $u$ accepts $v$). 
\begin{align*}  
 \kappa_f  &= \sum_{t=1}^T  \frac{r_v}{T} \Pr[\SF_{u,t}] \Big( \frac{\alp x^*_f}{r_v} + \frac{\beta y^*_f}{r_v}\Big) \\
 &\ge   \sum_{t=1}^T \frac{1}{T} \Big(1-\frac{1}{T}\Big)^{t-1} \Big(1-\frac{t-1}{T}\Big) \Big(\alp x^*_f+ \beta y^*_f\Big) \sim \frac{\alp x^*_f+ \beta y^*_f}{e}
   \end{align*}
 The last term is obtained after taking $T \rightarrow \infty$. 
 
Let $\pro(\alp, \beta)$ be the expected total profit obtained by $\lpalg(\alp, \beta)$. By linearity of expectation, we have $\pro(\alp, \beta) \ge \sum_{f \in E}\frac{1}{e}(\alp x^*_f+ \beta y^*_f\Big) p_e w_e $. From Lemma~\ref{lem:LP}, we know that the expected profit in offline optimal is upper bounded by $\sum_{f \in E} x_f^* p_e w_e$. Thus we claim that $\lpalg(\alp, \beta)$ achieves a ratio at least $\alp/e$ on the profit. Similarly, we can argue that $\lpalg(\alp, \beta)$ achieves a ratio at least $\beta/e$ on the fairness.  
\end{proof}

\section{Hardness Results}

The model in our paper has two objectives which complicate the hardness analysis. 
To simplify it, we focus only on those non-adaptive algorithms. We characterize a non-adaptive algorithm as $\z=\{z_f| f \in E_f\}$ where each $z_f \in [0,1]^{N_v}$ ($N_v$ is the size of $E_v$,  the set of edges incident to $v$) such that $\sum_{f \in E_v}z_{f} \le 1$. A non-adaptive algorithm parameterized with $\z$, denoted by  $\NALG(\z)$ will  sample an assignment $f =(u,v)\in E_v$ with probability $z_{f}$ upon the arrival of $v$, and assign $v$ to $u$ if $u$ is available. Note that our LP-based parameterized algorithm $\lpalg(\alp, \beta)$ can be viewed as a specific non-adaptive with $z_f=\alp x^*_f/r_v +\beta y^*_f/r_v$ for each $f$.

\begin{example}\label{exam:hard}
Consider a star graph where $U=\{u\}$ with $\Del_u=1$, $V=\{v_0,v_1,\ldots, v_K\}$ and $T=K+1$. 
We use $j$ to denote $v_j$ and edge $f_j=(u,v_j)$ when the context is clear. Let $w_j=1$ and $r_j=1$ for all $j=0,1,\ldots, K$. In other words, our star graph is unweighted and the arrival distributions are uniform. Set $p_j=1$ for $j=0$ and $p_j=\ep$ for each $j \in [K]$. Let $\OPTP$ and $\OPTF$ denote the optimal LP values of \LP-\eqref{obj-1} and \LP-\eqref{obj-2} on this example respectively. We can verify that: (1) $\OPTP=1$, where there is a unique optimal solution $x_0=1$ and $x_j=0$ for all $j \in [K]$; (2) $\OPTF=\frac{\ep}{K+\ep}$, where there is a unique optimal solution $x_0=\frac{\ep}{\ep+K}$ and $x_j=\frac{1}{\ep+K}$ for all $j \in [K]$. 
\end{example}

Consider a given non-adaptive algorithm $\NALG(\z)$ , with $\z=\{z_0,z_1,\ldots, z_K\}$ on the above example. Let $P(\z)$ and $F(\z)$ be the expected total profit and fairness achieved by $\NALG(\z)$. Set $z=\sum_{j=0}^K z_j$. Note that $0\le z \le 1$. 

\begin{lemma}\label{lem:hard-cr}
The sum of competitive ratios achieved by $\NALG(\z)$ for the profit and fairness on Example~\ref{exam:hard} is no larger than $1-1/e+2\ep$. 
\end{lemma}

\begin{proof}
\begin{align*}
P(\z)&=\frac{z_0}{T} \sum_{t=1}^T \Big(1-\frac{\sum_{j=0}^K z_j}{T} \Big)^{t-1}+ \sum_{j=1}^K \frac{z_j \ep}{T}   \sum_{t=1}^T \Big(1-\frac{\sum_{j=0}^K z_j}{T} \Big)^{t-1} 
\\
&=\frac{z_0}{z}\Big(1-\exp(-z) \Big)+\frac{(\sum_{j=1}^K z_j) \ep}{z}\Big(1-\exp(-z) \Big)~(T \rightarrow \infty)\\
&\le \frac{z_0}{z}\Big(1-\exp(-z) \Big)+\ep 
\end{align*}

From the above computation, we see that the expected sum of acceptances of assignments over all $j \in [K]$ during the online phase is $\frac{(\sum_{j=1}^K z_j) \ep}{z}\Big(1-\exp(-z) \Big)$. Thus the minimum ratio should be no larger than the average, from which we have
$$F(\z) \le \frac{(\sum_{j=1}^K z_j) \ep}{z K}\Big(1-\exp(-z) \Big)$$

Note that $\OPTP=1$ and $\OPTF=\frac{\ep}{K+\ep}$. Thus the sum of competitive ratios on the profit and fairness should be 
\begin{align*}
& \frac{P(\z)}{\OPTP}+\frac{F(\z)}{\OPTF} \\
& \le  \frac{z_0}{z}\Big(1-\exp(-z) \Big)+\ep+\frac{(\sum_{j=1}^K z_j) }{z} \frac{K+\ep}{K}\Big(1-\exp(-z) \Big) \\
& \le 1-\exp(-z)+2 \ep \le 1-1/e+2\ep
\end{align*}
\end{proof}
 Notice that Lemma~\ref{lem:hard-cr} immediately implies our main Theorem~\ref{thm:hard}.

\begin{figure*}[!t]
	\centering
	\begin{subfigure}
		\centering
		\includegraphics[width=0.21\textwidth]{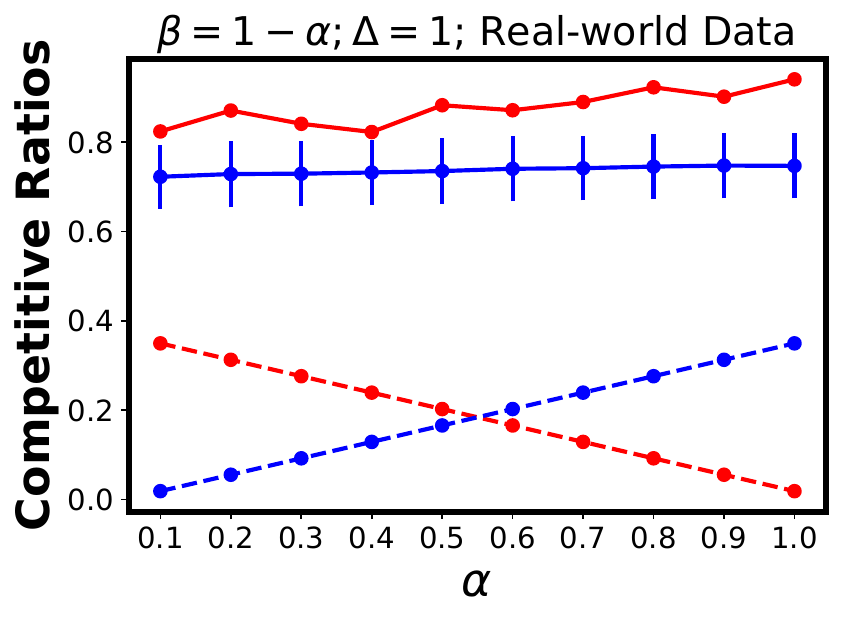}
	\end{subfigure}
	\begin{subfigure}
		\centering
		\includegraphics[width=0.21\textwidth]{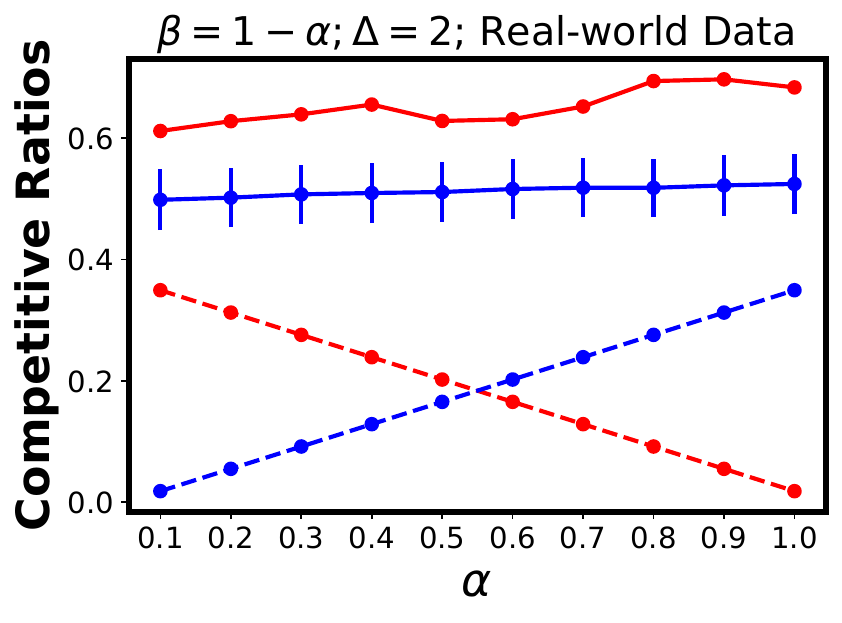}
	\end{subfigure}
	\begin{subfigure}
		\centering
		\includegraphics[width=0.21\textwidth]{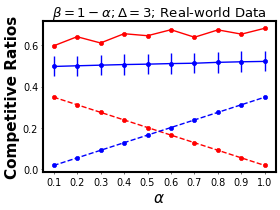}
	\end{subfigure}
	\begin{subfigure}
		\centering
		\includegraphics[width=0.15\textwidth]{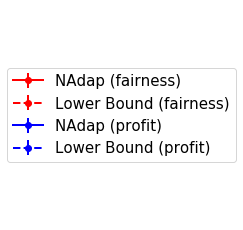}
	\end{subfigure}
\caption{Real dataset: competitive ratios for profit and fairness with different values of $\alpha$ and $\beta$ with $\alpha + \beta = 1$. $|U| = 48, |V| = 24, T = 359$. $\Delta=1$ (Left), $\Delta=2$ (Middle), and $\Delta=3$ (Right).}
\label{fig:profit_fairness_crs}
\end{figure*}

\begin{figure*}[!t]
	\centering
	\begin{subfigure}
		\centering
		\includegraphics[width=0.21\textwidth]{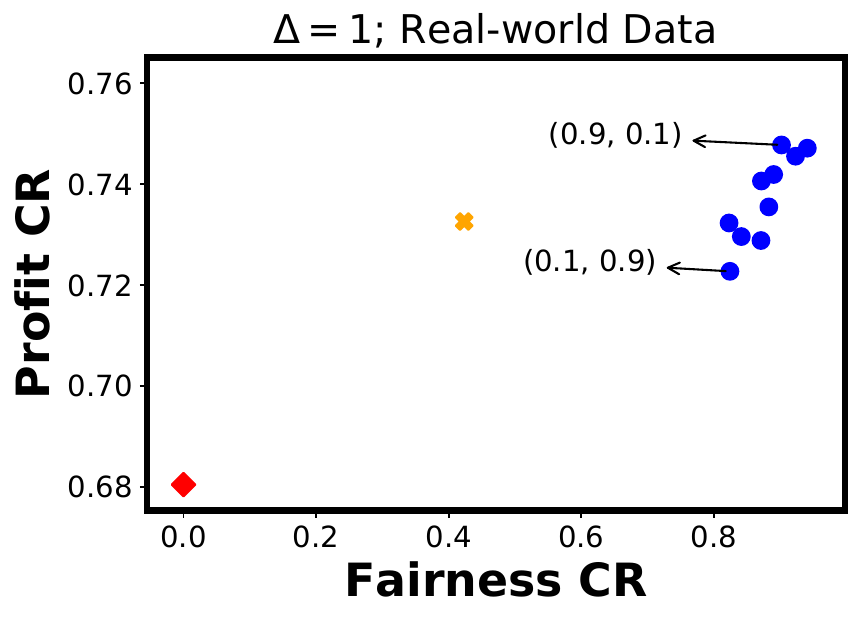}
	\end{subfigure}
	\begin{subfigure}
		\centering
		\includegraphics[width=0.21\textwidth]{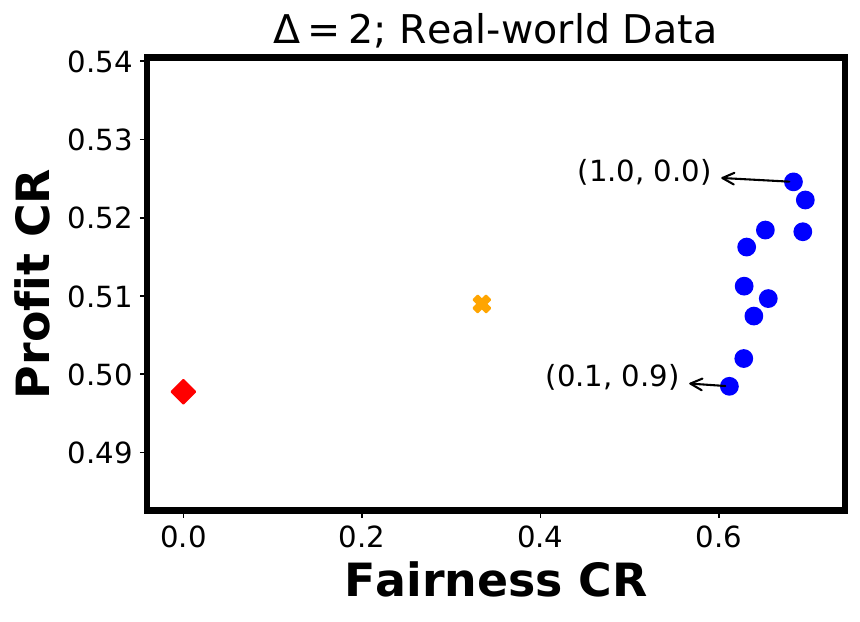}
	\end{subfigure}
	\begin{subfigure}
		\centering
		\includegraphics[width=0.21\textwidth]{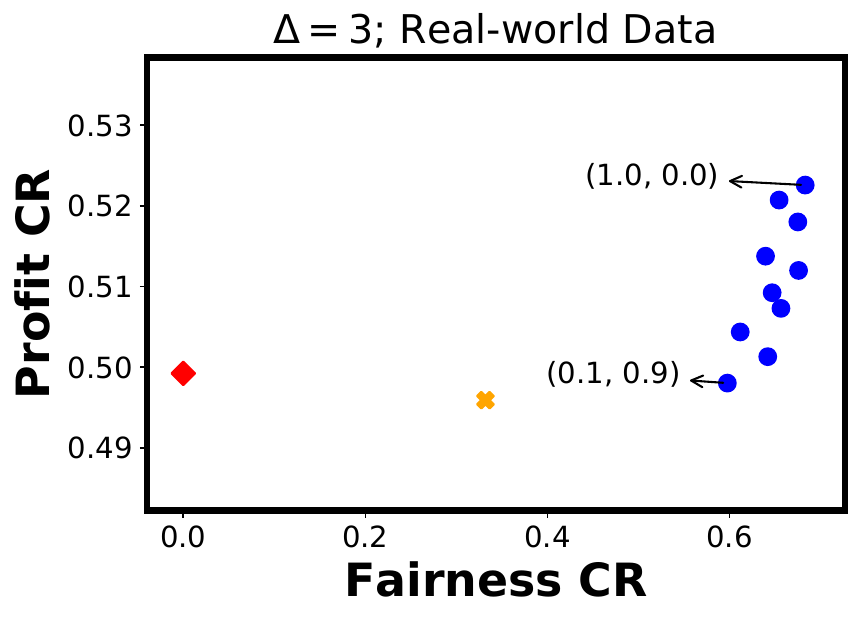}
	\end{subfigure}
	\begin{subfigure}
		\centering
		\includegraphics[width=0.15\textwidth]{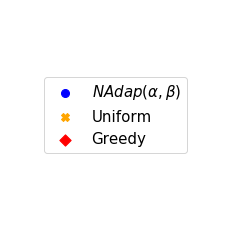}
	\end{subfigure}
\caption{Real dataset: comparison of performances of \lpalg against \gre and \uni. $|U| = 48, |V| = 24, T = 359$. $\Delta=1$ (Left), $\Delta=2$ (Middle), and $\Delta=3$ (Right).}
\label{fig:scatter}
\end{figure*}

\section{Experiments}\label{sec:exp}

\begin{figure*}
	\centering
	\begin{subfigure}
		\centering
		\includegraphics[width=0.22\textwidth]{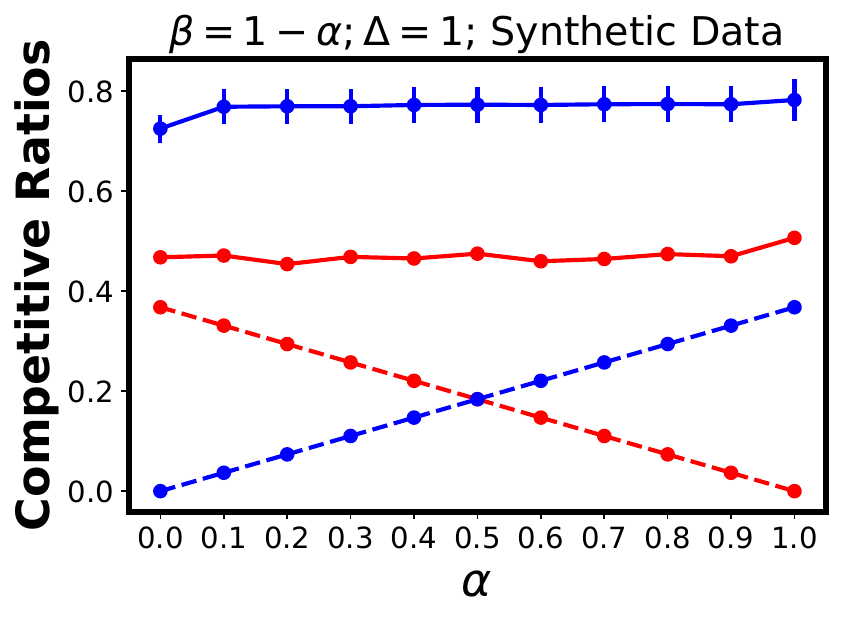}%
	\end{subfigure}
	\begin{subfigure}
		\centering
		\includegraphics[width=0.22\textwidth]{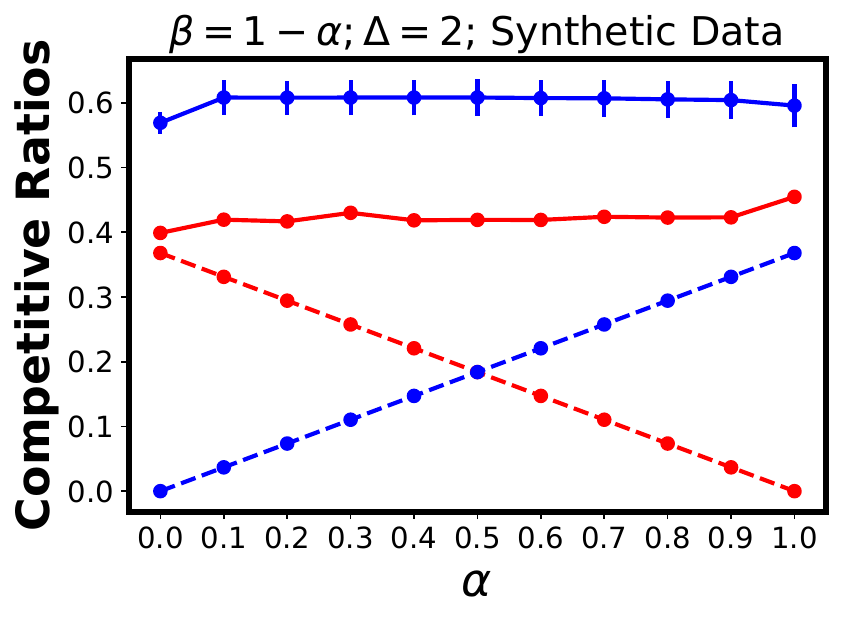}%
	\end{subfigure}
	\begin{subfigure}
		\centering
		\includegraphics[width=0.22\textwidth]{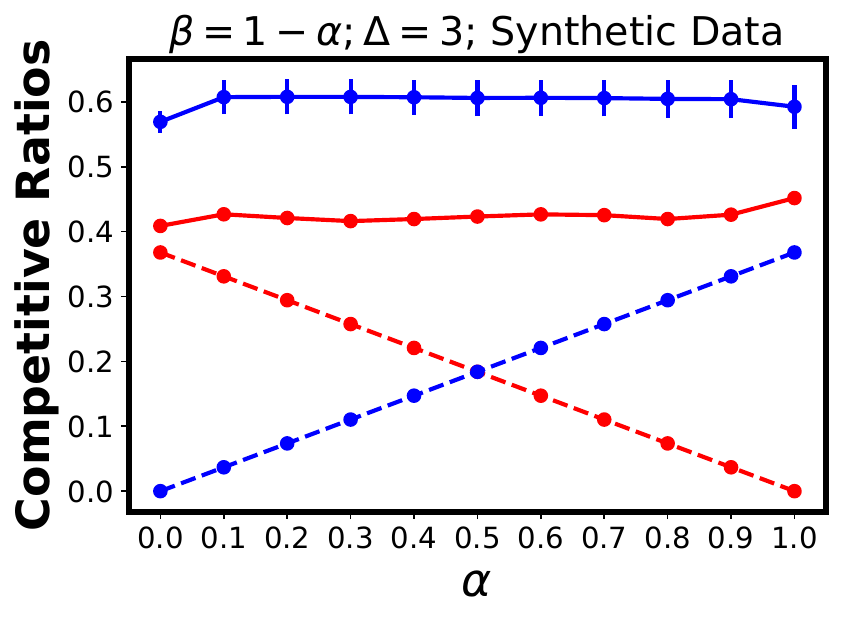}%
	\end{subfigure}
	\begin{subfigure}
			\centering
			\includegraphics[width=0.15\textwidth]{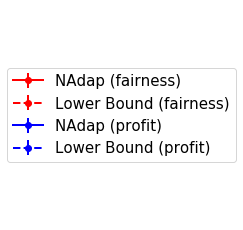}%
	\end{subfigure}
\caption{Synthetic dataset: competitive ratios for profit and fairness with different values of $\alpha$ and $\beta$ with $\alpha + \beta = 1$. $|U| = 100, |V| = 50, T = 700 , p_f \in [0.5,1]$. $\Delta=1$ (Left), $\Delta=2$ (Middle), and $\Delta=3$ (Right).}
\label{fig:profit_fairness_crs_synthetic_new_params}
\end{figure*}

\begin{figure*}
	\centering
	\begin{subfigure}
		\centering
		\includegraphics[width=0.22\textwidth]{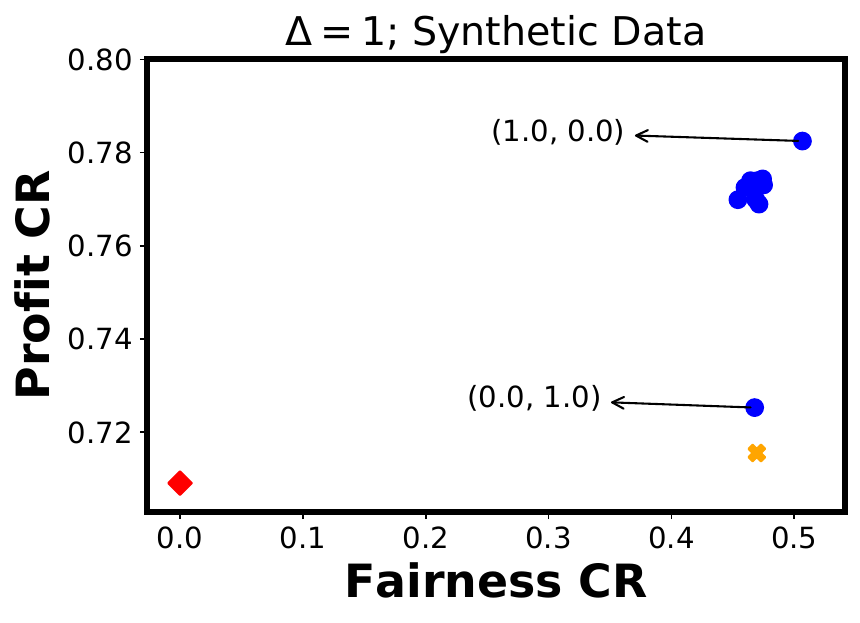}%
	\end{subfigure}
	\begin{subfigure}
		\centering
		\includegraphics[width=0.22\textwidth]{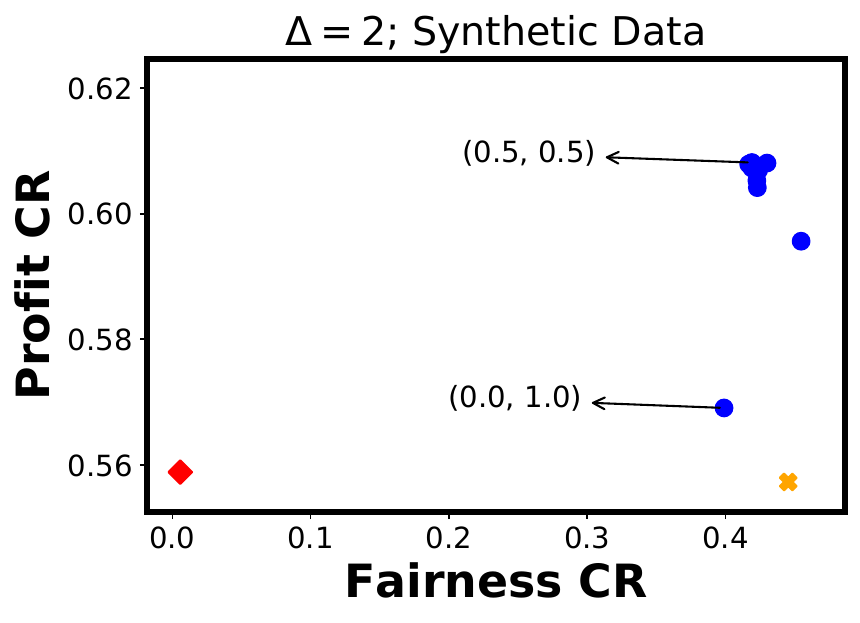}%
	\end{subfigure}
	\begin{subfigure}
		\centering
		\includegraphics[width=0.22\textwidth]{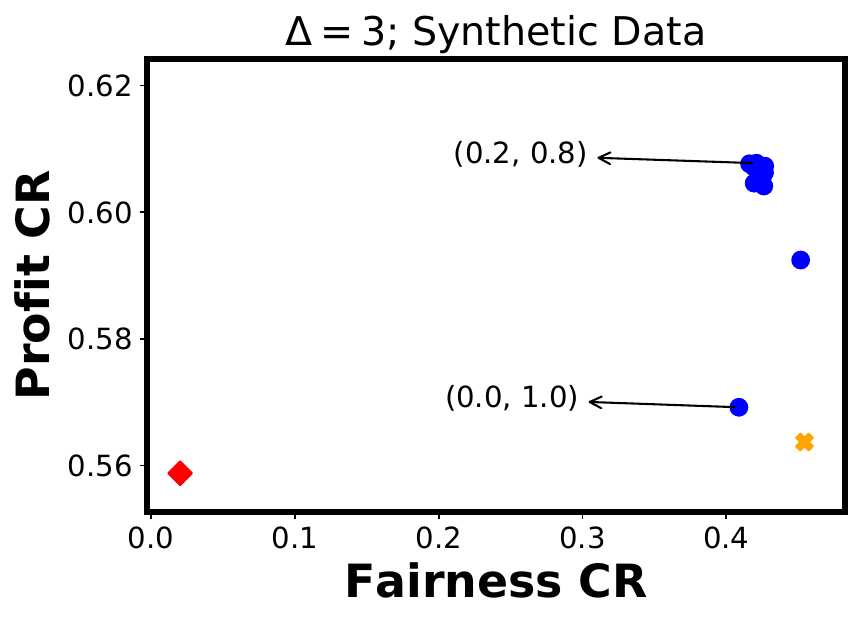}%
	\end{subfigure}
	\begin{subfigure}
		\centering
		\includegraphics[width=0.15\textwidth]{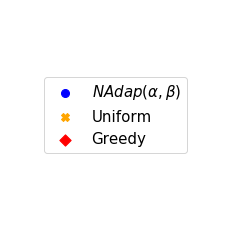}%
	\end{subfigure}
\caption{Synthetic dataset: comparison of performances of \lpalg against \gre and \uni. $|U| = 100, |V| = 50, T = 700 , p_f \in [0.5,1]$. $\Delta=1$ (Left), $\Delta=2$ (Middle), and $\Delta=3$ (Right).}
\label{fig:scatter_synthetic_new_params}
\end{figure*}

\xhdr{Real-world dataset processing}. 
We now provide experimental validation of our algorithms.  We use a real-world dataset: the New York City yellow cabs dataset\footnote{\url{http://www.andresmh.com/nyctaxitrips/}}\footnote{\url{https://tinyurl.com/y56gonvs}} which contains taxi trip records from Manhattan, Brooklyn, and Queens during the year 2013. The dataset is split into $12$ months. For each month we have a the record of a completed trip that contains an anonymized hash of a driver's license, the start and end location of the trip (latitude and longitude), time and date when the trip was initiated and ended, time taken to complete the trip, distance of the trip (in miles) and additional attributes such as number of passengers and registration number. Demographics of the drivers and riders are not known; however, to demonstrate our methodology, we randomly assign a ``disadvantaged'' or ``advantaged'' race status to all requesters such that ratio of disadvantaged riders to advantaged riders is 1:2, which roughly matches the racial demographics of the city.\footnote{\url{https://tinyurl.com/y7q8dppr}} For drivers, we randomly assign race such that the ratio of disadvantaged to advantaged drivers is 3:1.\footnote{\url{https://tinyurl.com/y27anseh}}\footnote{For simplicity, we ignore other potential discrimination factors such as gender and disability.} For the purpose of demonstration, we consider race to be binary; however, in reality, there would be many other races, which our model can easily incorporate, given that the distribution of $p_f$ is known or can be learned.

We consider the rush hour of 7--8PM; in our dataset, this period typically has the largest number of trips per hour of the day. On Jan. 31st 2013, a total of \num{35109} trips were completed by a total of \num{10814} drivers.  One driver might make multiple trips during this period. However, here we assume that once a driver starts the trip, they are out of the system. So, we assume that all \num{10814} drivers are present at the start and need to serve a total of \num{35109} trips (\ie $T = 35,\!109$). This is consistent with peak hour assumptions ($T \gg |U|$). Similar to~\citen{xuAAAI18}, we focus on longitudes from $-73^{\circ}$ to $-75^{\circ}$ and latitudes from $40.4^{\circ}$ to $40.95^{\circ}$, both with a step size of $0.05$. We then assign an index to each grid and every location with latitudes and/or longitude in a particular grid is represented by the index of that grid.

We construct our graph $G=(U,V,E)$ as follows. Each $u \in U$ represents a driver type which has attributes of the starting location and race. Each $v \in V$ represents a request type which has attributes of the starting location, ending location, and race. We downsample from all driver and request types such that $|U|=48$ and $|V|=24$. For each $u$, we set a uniform cancellation quota $\Del$; across three experiments, we vary $\Del \in \{1,2,3\}$. For each $v$, we generate a random value $r_v \sim \mathcal{N}(15, 1)$ (normal distribution) and assign the value as its arrival rate. Set $T=\sum_{v \in V} r_v$. We assume that the driver's starting coordinates belong to the same bin as that of the start coordinates of the trip. An edge exists between a driver type $u$ and a request type $v$, iff the starting coordinates of $v$ is in the same bin as that of $u$. We set the existing probabilities mainly based on the combination of driver and rider's races. When both driver and rider belong to the advantaged class, we assign $p_f = 0.6$; if driver and rider are both members of the disadvantaged class then we assign $p_f = 0.3$ and in all other cases $p_f = 0.1$. These probabilities are then scaled up by a factor $\kappa$, \ie $p_f$ = $\kappa$ + $(1 - \kappa)\cdot p_f$. In our experiments we set $\kappa = 0.5$. The profit associated with each edge ($w_f$) is defined as the normalized distance of the request type $v$ (so $0 \le w_f \le 1$).

\xhdr{Synthetic dataset}.
To generate a synthetic dataset, we fix $|U| = 100, |V| = 50, T = 700$, and sample arrival rates randomly from a multinomial distribution to ensure they sum up to $T(=700)$. Between each pair $u \in U$ and $v \in V$, an edge exists with probability $0.1$ and does not exist with probability $0.9$. The probabilities $p_f$ are randomly sampled such that $p_f \in [0.5, 1]$ and the profit values ($w_f$) associated with each edge are randomly sampled such that $w_f \in [0,1]$. We assign each $u$ the same cancellation quota $\Delta$, and vary $\Delta \in \{1,2,3\}$ across three experiments.

\xhdr{Algorithms}. We run \lpalg for $5000$ iterations and take the average values over these runs to be the expectations. We use profit computed via \LP-\eqref{obj-1} as the benchmark and use it to calculate the competitive ratio of profit for \lpalg. Similarly, we use \LP-\eqref{obj-2} as the benchmark for fairness and use it to calculate the competitive ratio of fairness for \lpalg. 
For \gre and \uni, say a request of type $v$ arrives at time $t$ and let $E_{v,t}$ be the set of available assignments at $t$ and let $E_v$ be all possible assignments for $v$. In \uni, we sample an assignment uniformly from $E_v$ and check if it exists in $E_{v,t}$. If it exists, then we make the assignment; otherwise, we reject the request. In \gre, we select the assignment with the highest $p_f$ value among $E_{v,t}$.

\xhdr{Results on the real and synthetic datasets}.
Figures~\ref{fig:profit_fairness_crs} and~\ref{fig:scatter} show results on the real dataset. We see that the solid lines (performance of \lpalg) always stay above the dotted ones (theoretical lower bounds). 
Note that a higher $\beta$ means that $\lpalg$ is guided more heavily by $\y^{*}$ and a higher $\alpha$ means $\lpalg$ is guided more heavily by $\x^{*}$. For the cases when this trend is not apparent (\eg for $\Delta = 3$ in Figure~\ref{fig:profit_fairness_crs}), we posit that the optimal solutions ($\x^{*}$ and $\y^{*}$) are highly correlated. Figure~\ref{fig:scatter} shows that the effectiveness and flexibility of \lpalg on both objectives compared to \gre and \uni. 
Specifically, \gre is very weak and dominated by almost all variants of \lpalg on both objectives. This is expected since \gre as a heuristic is short-sighted and does not make assignments based on the what kind of requests the system expects to get in the future. \uni seems to achieve relatively high fairness, which makes sense due to the nature of \uni. However, \lpalg provides a wide range of options and some of them can beat \uni on both objectives as well.  Figures~\ref{fig:profit_fairness_crs_synthetic_new_params} and~\ref{fig:scatter_synthetic_new_params} show results on the synthetic dataset. The trends are more consistent since all values here are randomly picked. We see that even in this case \lpalg is able to do better than \gre and \uni on both the objectives. 

\section{Conclusion}

In this work, we present a flexible approach for matching requests to drivers during peak hours that balances fairness and profit. Our proposed approach allows the system designer to specify how fair and how profitable they want the system to be via two separate parameters. We take a nuanced view of the ridesharing market and model the problem as an online bipartite matching problem with stochastic rewards. One highlight is the introduction of existence probability on each edge, which captures the potential acceptance rate for each pair of driver and rider types. We present an LP-based algorithm that dynamically assigns requests to drivers. Extensive experimental results on both of the real and synthetic datasets show that our proposed approach is not only above the theoretical lower bounds but also can beat natural approaches such as \gre and \uni on both objectives. Our work presents many interesting directions for future research. For example, we propose a non-adaptive algorithm in which once a request is rejected (either by the system or the driver), it will not be assigned again. A possible direction for future work could be to re-assign a rejected request. Finally, we assume that drivers do not appear again in the system for the peak hour, however, in reality, that might not always be the case. We hope that our work will encourage the community to look at such problems.

\section{Acknowledgements}
We would like to thank anonymous reviewers for their helpful comments. Nanda and Dickerson were supported by NSF CAREER Award IIS-1846237 and DARPA SI3-CMD Award S4761. Srinivasan was supported in part by NSF CNS-1010789, CCF-1422569, and CCF-1749864, and by research awards from Adobe, Amazon, and Google. Dickerson and Srinivasan were both supported by a gift from Google and a seed grant from the Maryland Transportation Institute. Work and the corresponding funding was done when Sankararaman was affiliated with University of Maryland, College Park.

{\small
\bibliographystyle{aaai}
\bibliography{stable_ref}

\begin{thebibliography}{}

\bibitem[\protect\citeauthoryear{Aggarwal \bgroup et al\mbox.\egroup
  }{2014}]{aggarwal2014}
Aggarwal, G.; Cai, Y.; Mehta, A.; and Pierrakos, G.
\newblock 2014.
\newblock Biobjective online bipartite matching.
\newblock In {\em WINE},  218--231.
\newblock Springer.

\bibitem[\protect\citeauthoryear{Ashlagi \bgroup et al\mbox.\egroup
  }{2019}]{ashlagi2019edge}
Ashlagi, I.; Burq, M.; Dutta, C.; Jaillet, P.; Sholley, C.; and Saberi, A.
\newblock 2019.
\newblock Edge weighted online windowed matching.
\newblock In {\em Proceedings of EC},  729--742.

\bibitem[\protect\citeauthoryear{Banerjee, Freund, and
  Lykouris}{2017}]{Banerjee-ec-17}
Banerjee, S.; Freund, D.; and Lykouris, T.
\newblock 2017.
\newblock Pricing and optimization in shared vehicle systems: An approximation
  framework.
\newblock EC '17,  517--517.

\bibitem[\protect\citeauthoryear{Banerjee, Johari, and
  Riquelme}{2016}]{Banerjee:2016}
Banerjee, S.; Johari, R.; and Riquelme, C.
\newblock 2016.
\newblock Dynamic pricing in ridesharing platforms.
\newblock {\em SIGecom Exch.} 15(1):65--70.

\bibitem[\protect\citeauthoryear{Bansal \bgroup et al\mbox.\egroup
  }{2012}]{bansal2012lp}
Bansal, N.; Gupta, A.; Li, J.; Mestre, J.; Nagarajan, V.; and Rudra, A.
\newblock 2012.
\newblock When {LP} is the cure for your matching woes: Improved bounds for
  stochastic matchings.
\newblock {\em Algorithmica} 63(4):733--762.

\bibitem[\protect\citeauthoryear{Bateni \bgroup et al\mbox.\egroup
  }{2016}]{bateni2018fair}
Bateni, M.~H.; Chen, Y.; Ciocan, D.; and Mirrokni, V.
\newblock 2016.
\newblock Fair resource allocation in a volatile marketplace.
\newblock EC '16,  819--819.

\bibitem[\protect\citeauthoryear{Bei and Zhang}{2018}]{BeiZ18}
Bei, X., and Zhang, S.
\newblock 2018.
\newblock Algorithms for trip-vehicle assignment in ride-sharing.
\newblock AAAI '18,  3--9.

\bibitem[\protect\citeauthoryear{Bimpikis, Candogan, and
  Saban}{2017}]{bimpikis2019spatial}
Bimpikis, K.; Candogan, O.; and Saban, D.
\newblock 2017.
\newblock Spatial pricing in ride-sharing networks.
\newblock NetEcon '17,  5:1--5:1.

\bibitem[\protect\citeauthoryear{Brubach \bgroup et al\mbox.\egroup
  }{2017}]{BSSX17}
Brubach, B.; Sankararaman, K.~A.; Srinivasan, A.; and Xu, P.
\newblock 2017.
\newblock Attenuate locally, win globally: An attenuation-based framework for
  online stochastic matching with timeouts.
\newblock In {\em Proceedings of AAMAS},  1223--1231.

\bibitem[\protect\citeauthoryear{Cook}{2019}]{web-disability}
Cook, G.
\newblock 2019.
\newblock Woman says uber driver denied her ride because of her wheelchair.
\newblock \url{https://tinyurl.com/yxea2r97}.
\newblock Accessed: 2019-06-12.

\bibitem[\protect\citeauthoryear{Dickerson \bgroup et al\mbox.\egroup
  }{2018a}]{xuAAAI18}
Dickerson, J.~P.; Sankararaman, K.~A.; Srinivasan, A.; and Xu, P.
\newblock 2018a.
\newblock Allocation problems in ride-sharing platforms: Online matching with
  offline reusable resources.
\newblock AAAI '18,  1007--1014.

\bibitem[\protect\citeauthoryear{Dickerson \bgroup et al\mbox.\egroup
  }{2018b}]{dickerson2018assigning}
Dickerson, J.~P.; Sankararaman, K.~A.; Srinivasan, A.; and Xu, P.
\newblock 2018b.
\newblock Assigning tasks to workers based on historical data: Online task
  assignment with two-sided arrivals.
\newblock In {\em Proceedings of the 17th International Conference on
  Autonomous Agents and MultiAgent Systems}, AAMAS '18,  318--326.

\bibitem[\protect\citeauthoryear{Esfandiari, Korula, and
  Mirrokni}{2016}]{esfandiari2016bi}
Esfandiari, H.; Korula, N.; and Mirrokni, V.
\newblock 2016.
\newblock Bi-objective online matching and submodular allocations.
\newblock In {\em Advances in Neural Information Processing Systems},
  2739--2747.

\bibitem[\protect\citeauthoryear{Fain, Munagala, and Shah}{2018}]{fain2018fair}
Fain, B.; Munagala, K.; and Shah, N.
\newblock 2018.
\newblock Fair allocation of indivisible public goods.
\newblock In {\em Proceedings of EC},  575--592.
\newblock ACM.

\bibitem[\protect\citeauthoryear{Ghodsi \bgroup et al\mbox.\egroup
  }{2011}]{ghodsi2011dominant}
Ghodsi, A.; Zaharia, M.; Hindman, B.; Konwinski, A.; Shenker, S.; and Stoica,
  I.
\newblock 2011.
\newblock Dominant resource fairness: Fair allocation of multiple resource
  types.
\newblock In {\em NSDI}, volume~11,  24--24.

\bibitem[\protect\citeauthoryear{Grandoni, Ravi, and
  Singh}{2009}]{grandoni2009}
Grandoni, F.; Ravi, R.; and Singh, M.
\newblock 2009.
\newblock Iterative rounding for multi-objective optimization problems.
\newblock In {\em European Symposium on Algorithms},  95--106.
\newblock Springer.

\bibitem[\protect\citeauthoryear{Griffin}{2019}]{web-guided-dog}
Griffin, K.
\newblock 2019.
\newblock Ont. woman says uber driver rejected her guide dog.
\newblock \url{https://tinyurl.com/y7sd268y}.
\newblock Accessed: 2019-06-12.

\bibitem[\protect\citeauthoryear{Houser}{2018}]{web-race}
Houser, K.
\newblock 2018.
\newblock Uber and {L}yft still allow racist behavior, but not as much as taxi
  services.
\newblock \url{https://tinyurl.com/yyln8e3v}.
\newblock Accessed: 2019-06-12.

\bibitem[\protect\citeauthoryear{Kanoria and Qian}{2019}]{kanoria2019near}
Kanoria, Y., and Qian, P.
\newblock 2019.
\newblock Near optimal control of a ride-hailing platform via mirror
  backpressure.
\newblock {\em arXiv preprint arXiv:1903.02764}.

\bibitem[\protect\citeauthoryear{Kash, Procaccia, and Shah}{2014}]{kash2014no}
Kash, I.; Procaccia, A.~D.; and Shah, N.
\newblock 2014.
\newblock No agent left behind: Dynamic fair division of multiple resources.
\newblock {\em JAIR} 51:579--603.

\bibitem[\protect\citeauthoryear{Laurence}{2019}]{web-female}
Laurence, L.
\newblock 2019.
\newblock Uber driver fired for refusing to drive woman to get abortion.
\newblock \url{https://tinyurl.com/y2cadftk}.
\newblock Accessed: 2019-06-12.

\bibitem[\protect\citeauthoryear{Lesmana, Zhang, and
  Bei}{2019}]{lesmana2019neurips}
Lesmana, N.~S.; Zhang, X.; and Bei, X.
\newblock 2019.
\newblock Balancing efficiency and fairness in on-demand ridesourcing.
\newblock In Wallach, H.; Larochelle, H.; Beygelzimer, A.; d\textquotesingle
  Alch\'{e}-Buc, F.; Fox, E.; and Garnett, R., eds., {\em Advances in Neural
  Information Processing Systems 32}. Curran Associates, Inc.
\newblock  5310--5320.

\bibitem[\protect\citeauthoryear{Li \bgroup et al\mbox.\egroup
  }{2018}]{DBLP:conf/kdd/LiFWSYL18}
Li, Y.; Fu, K.; Wang, Z.; Shahabi, C.; Ye, J.; and Liu, Y.
\newblock 2018.
\newblock Multi-task representation learning for travel time estimation.
\newblock KDD '18,  1695--1704.

\bibitem[\protect\citeauthoryear{Lin \bgroup et al\mbox.\egroup
  }{2018}]{lin2018efficient}
Lin, K.; Zhao, R.; Xu, Z.; and Zhou, J.
\newblock 2018.
\newblock Efficient large-scale fleet management via multi-agent deep
  reinforcement learning.
\newblock  1774--1783.

\bibitem[\protect\citeauthoryear{Lowalekar, Varakantham, and
  Jaillet}{2018}]{Patrick-18-JAI}
Lowalekar, M.; Varakantham, P.; and Jaillet, P.
\newblock 2018.
\newblock Online spatio-temporal matching in stochastic and dynamic domains.
\newblock {\em Artificial Intelligence} 261:71 -- 112.

\bibitem[\protect\citeauthoryear{Ma, Fang, and Parkes}{2019}]{ma2018spatio}
Ma, H.; Fang, F.; and Parkes, D.~C.
\newblock 2019.
\newblock Spatio-temporal pricing for ridesharing platforms.
\newblock EC '19,  583--583.

\bibitem[\protect\citeauthoryear{Parkes, Procaccia, and
  Shah}{2015}]{parkes2015beyond}
Parkes, D.~C.; Procaccia, A.~D.; and Shah, N.
\newblock 2015.
\newblock Beyond dominant resource fairness: Extensions, limitations, and
  indivisibilities.
\newblock {\em ACM TEAC} 3(1):3.

\bibitem[\protect\citeauthoryear{Paul}{2018}]{web-cancel}
Paul, K.
\newblock 2018.
\newblock There's a quiet battle of wills between uber drivers and customers
  over who cancels first.
\newblock \url{https://tinyurl.com/y8jx6dol}.
\newblock Accessed: 2019-06-12.

\bibitem[\protect\citeauthoryear{Ravi \bgroup et al\mbox.\egroup
  }{1993}]{ravi1993many}
Ravi, R.; Marathe, M.~V.; Ravi, S.; Rosenkrantz, D.~J.; and Hunt~III, H.~B.
\newblock 1993.
\newblock Many birds with one stone: Multi-objective approximation algorithms.
\newblock STOC '93,  438--447.
\newblock Citeseer.

\bibitem[\protect\citeauthoryear{Singer and Mittal}{2013}]{singer2013pricing}
Singer, Y., and Mittal, M.
\newblock 2013.
\newblock Pricing mechanisms for crowdsourcing markets.
\newblock WWW '13,  1157--1166.

\bibitem[\protect\citeauthoryear{Singla and Krause}{2013}]{singla2013truthful}
Singla, A., and Krause, A.
\newblock 2013.
\newblock Truthful incentives in crowdsourcing tasks using regret minimization
  mechanisms.
\newblock WWW '13,  1167--1178.

\bibitem[\protect\citeauthoryear{S\"{u}hr \bgroup et al\mbox.\egroup
  }{2019}]{suhr2019}
S\"{u}hr, T.; Biega, A.~J.; Zehlike, M.; Gummadi, K.~P.; and Chakraborty, A.
\newblock 2019.
\newblock Two-sided fairness for repeated matchings in two-sided markets: A
  case study of a ride-hailing platform.
\newblock KDD '19,  3082--3092.

\bibitem[\protect\citeauthoryear{Tong \bgroup et al\mbox.\egroup
  }{2016a}]{tong2016vldb}
Tong, Y.; She, J.; Ding, B.; Chen, L.; Wo, T.; and Xu, K.
\newblock 2016a.
\newblock Online minimum matching in real-time spatial data: experiments and
  analysis.
\newblock {\em Proceedings of the VLDB Endowment} 9(12):1053--1064.

\bibitem[\protect\citeauthoryear{Tong \bgroup et al\mbox.\egroup
  }{2016b}]{tong2016icde}
Tong, Y.; She, J.; Ding, B.; Wang, L.; and Chen, L.
\newblock 2016b.
\newblock Online mobile micro-task allocation in spatial crowdsourcing.
\newblock ICDE '16,  49--60.

\bibitem[\protect\citeauthoryear{Tong \bgroup et al\mbox.\egroup
  }{2017}]{tong2017flexible}
Tong, Y.; Wang, L.; Zhou, Z.; Ding, B.; Chen, L.; Ye, J.; and Xu, K.
\newblock 2017.
\newblock Flexible online task assignment in real-time spatial data.
\newblock {\em Proceedings of the VLDB Endowment} 10(11):1334--1345.

\bibitem[\protect\citeauthoryear{Wang, Fu, and
  Ye}{2018}]{DBLP:conf/kdd/WangFY18}
Wang, Z.; Fu, K.; and Ye, J.
\newblock 2018.
\newblock Learning to estimate the travel time.
\newblock KDD '18,  858--866.

\bibitem[\protect\citeauthoryear{Xu \bgroup et al\mbox.\egroup
  }{2018}]{xu2018large}
Xu, Z.; Li, Z.; Guan, Q.; Zhang, D.; Li, Q.; Nan, J.; Liu, C.; Bian, W.; and
  Ye, J.
\newblock 2018.
\newblock Large-scale order dispatch in on-demand ride-hailing platforms: A
  learning and planning approach.
\newblock KDD '18,  905--913.

\bibitem[\protect\citeauthoryear{Yao \bgroup et al\mbox.\egroup
  }{2018}]{Yao2018deep}
Yao, H.; Wu, F.; Ke, J.; Tang, X.; Jia, Y.; Lu, S.; Gong, P.; Ye, J.; and Li,
  Z.
\newblock 2018.
\newblock Deep multi-view spatial-temporal network for taxi demand prediction.
\newblock AAAI '18,  2588--2595.

\bibitem[\protect\citeauthoryear{Zhao \bgroup et al\mbox.\egroup
  }{2019}]{xu-aaai-19}
Zhao, B.; Xu, P.; Shi, Y.; Tong, Y.; Zhou, Z.; and Zeng, Y.
\newblock 2019.
\newblock Preference-aware task assignment in on-demand taxi dispatching: An
  online stable matching approach.
\newblock AAAI '19,  2245--2252.

\end{thebibliography}
}

\newpage 

\section{Proof of Lemma 3}

\emph{\textbf{Lemma 3}}
$$\Pr[B_{u,t} \le \Del_u-1 | A_{u,t}=0 ] \ge 1-\frac{t-1}{T}$$

\begin{proof}
Consider a given $f \in E_u$ and $\ell<t$.
\begin{align*}
&\E[X_{f,\ell} Y_{f,\ell} |A_{u,t}=0] =\E[X_{f,\ell} Y_{f,\ell} |A_{u,\ell}=0] \\
&=\frac{\Pr[X_{f,\ell}=Y_{f,\ell}=1, Z_{f,\ell}=0]}{\Pr[A_{u,\ell}=0]}\\
&=\frac{(r_v/T)(\alp x_f^*/r_v+ \beta y_f^*/r_v) (1-p_f)}{1-\sum_{f\in E_u} (r_v/T) (\alp x_f^*/r_v+ \beta y_f^*/r_v)p_f}\\
&=\frac{(r_v/T)(\alp x_f^*/r_v+ \beta y_f^*/r_v) (1-p_f)}{1-p_f+p_f\Big(1-\sum_{f\in E_u} (1/T) (\alp x_f^*+ \beta y_f^* ) \Big)}\\
& \le  (r_v/T)(\alp x_f^*/r_v+ \beta y_f^*/r_v) =\E[X_{f,\ell} Y_{f,\ell}]
\end{align*}
The last inequality follows from the fact that $\sum_{f \in E_u}  (1/T)(\alp x_f^*+ \beta y_f^* ) \le \Del_u/T<1$ according to our assumption. The above analysis implies that conditional expectation of each term  $X_{f,\ell} Y_{f,\ell}$ over $(A_{u,t}=0)$ is no more than the unconditional expectation. By applying a similar analysis in the second part of proof in Lemma ~\ref{lem:1}, we get our claim.  
\end{proof}

\begin{figure}[!ht]
	\begin{subfigure}
		\centering
		\includegraphics[width=0.6\textwidth]{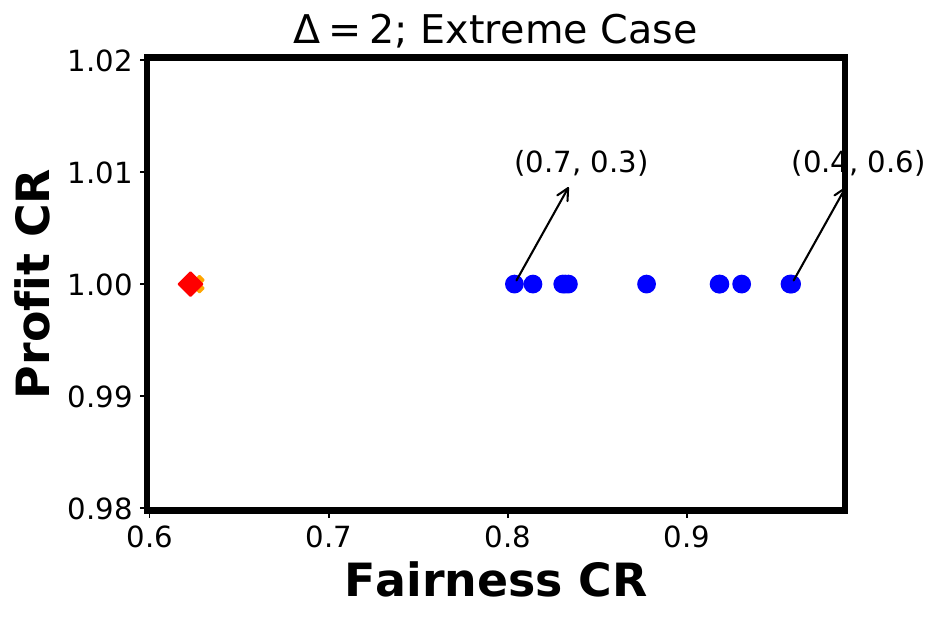}
	\end{subfigure}
	\begin{subfigure}
		\centering
		\includegraphics[width=0.6\textwidth]{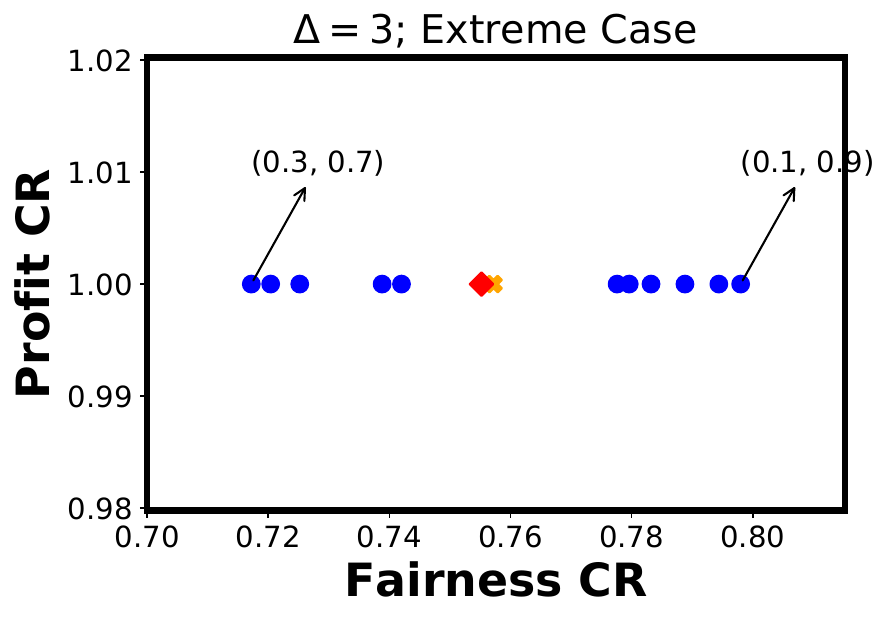}
	\end{subfigure}
	\begin{subfigure}
		\centering
		\includegraphics[width=0.4\textwidth]{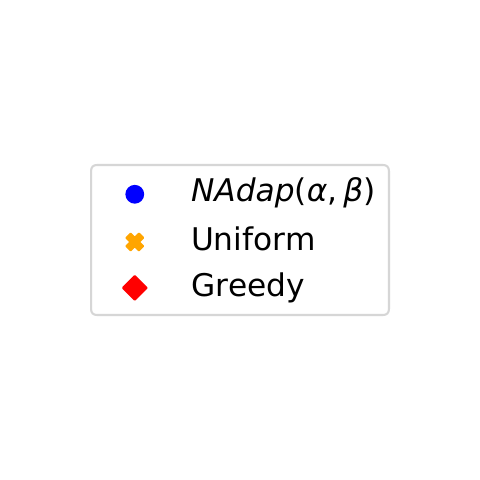}
	\end{subfigure}
\caption{Extreme Case: comparison of performances of \lpalg against \gre and \uni. $|U| = 1$, $|V| = T = 3$ with $p_0 = 1$ and $p_1, p_2 = 0.25$. All $w_f = 1$ and $r_v = 1$. $\Delta=2$ (Top), and $\Delta=3$ (Bottom).}
\label{fig:scatter_synthetic_extreme}
\end{figure}

\section{Peak Hour}

Based on the empirical distribution of the number of completed trips by hour of the day (Figure~\ref{fig:requests_distribution}) for January 2013, we see that 7PM - 8PM (19:00 - 20:00 hrs) has the maximum number of completed trips. Thus, we use this hour as the peak hour in all our experiments on the NYC taxi dataset.

\begin{figure}[!h]
\centering
\caption{Distribution of requests by hour of the day averaged over January.}
\includegraphics[width=0.7\textwidth]{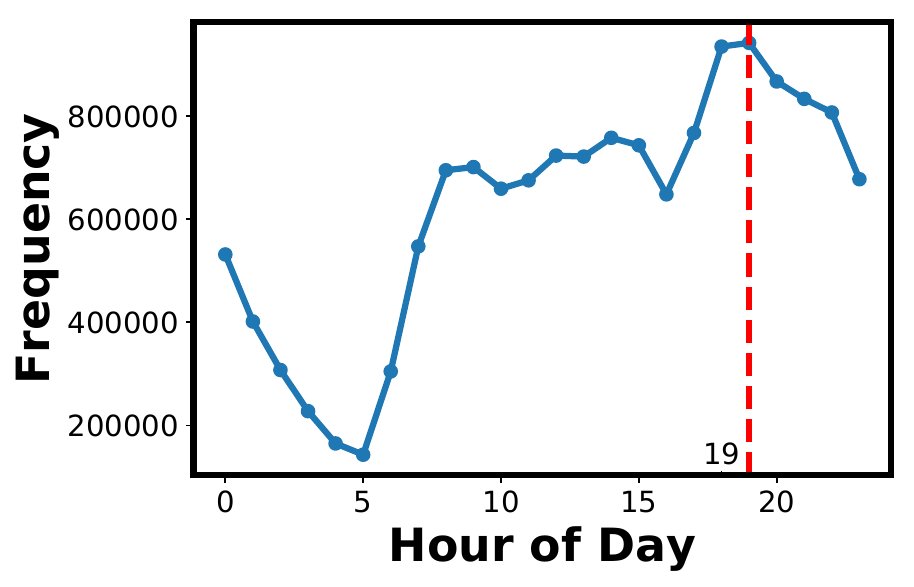}
\label{fig:requests_distribution}
\end{figure}

\section{Experiments - Extreme Case}

\begin{figure}[!ht]
	\begin{subfigure}
		\centering
		\includegraphics[width=0.6\textwidth]{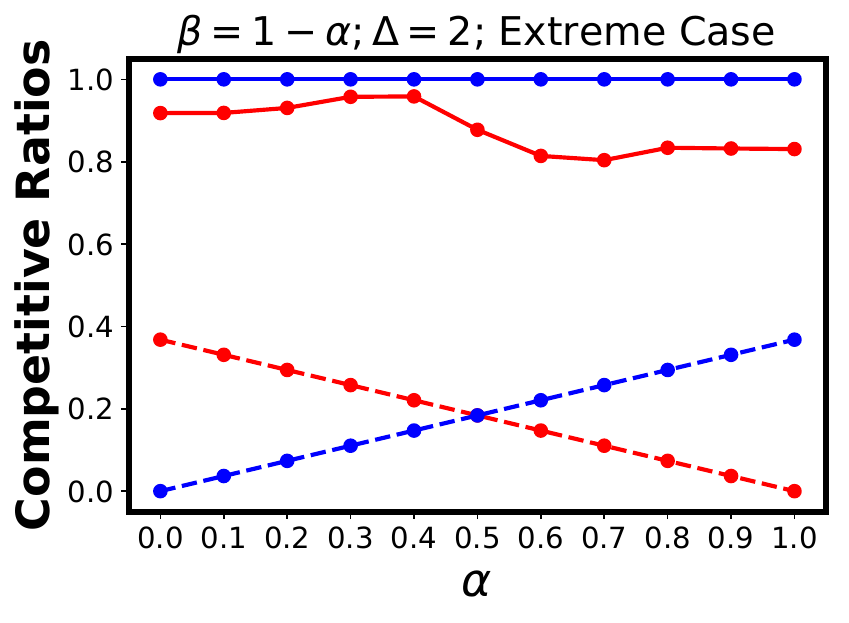}
	\end{subfigure}
	\begin{subfigure}
		\centering
		\includegraphics[width=0.6\textwidth]{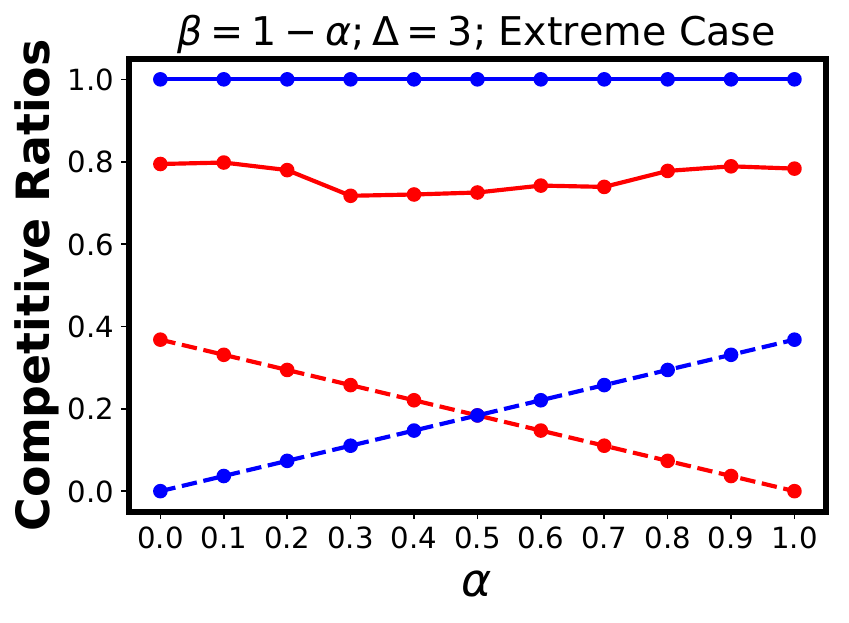}
	\end{subfigure}
	\begin{subfigure}
		\centering
		\includegraphics[width=0.4\textwidth]{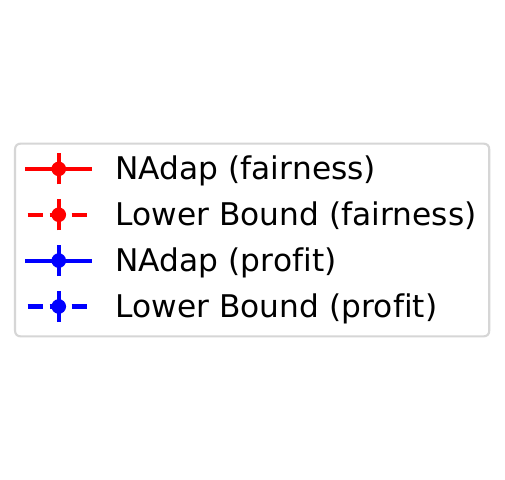}
	\end{subfigure}

\caption{Extreme Case: competitive ratios for profit and fairness with different values of $\alpha$ and $\beta$ with $\alpha + \beta = 1$. $|U| = 1$, $|V| = T = 3$ with $p_0 = 1$ and $p_1, p_2 = 0.25$. All $w_f = 1$ and $r_v = 1$. $\Delta=2$ (Top), and $\Delta=3$ (Bottom).}
\label{fig:profit_fairness_crs_synthetic_extreme}
\end{figure}

To illustrate the effect of $\alpha$ and $\beta$ on \lpalg, consider the extreme case when |U| = 1 ($U = {u_0}$) and |V| = 3 and $V = {v_0, v_1, v_2}$ (\ie there is only one driver in the system and 3 types of request come in during the pear hour). The arrival rate ($r_v$) for each request is 1, which gives $T = 3$. The edge probabilities are as follows, $p_{u_0, v_0} = 1$ and $p_{u_0, v_k} = 0.25$ for $k=1, 2$. For simplicity, let profit associated with each edge ($w_f$) be 1. In this scenario, if one were to maximize profit (via solving \LP-\eqref{obj-1}), a possible solution to $x_f$ (number of probes to an edge) could be $x_0 = 1$ and $x_1 = x_2 = 0$. In such a case, fairness of the system would be zero (since requests $v_1$ and $v_2$ would never be matched). However, if one were to maximize fairness (via solving \LP-\eqref{obj-2}), then the solution would ensure that all $x_i > 0$. Recall that by definition of fairness, \LP-\eqref{obj-2} would maximize the expected number of probes for the $v$ with minimum number of matched edges, which would ensure that all types of requests are matched over time. Hence this case gives us very different solutions to \LP-\eqref{obj-1} ($x^{*}$) and \LP-\eqref{obj-2} ($y^{*}$). Thus, running \lpalg with different $\alpha$ and $\beta$ values should result in different matchings and thus different fairness and porfit values. The results are discussed in the next section.

\subsection{Results}

It is important to note that we must set $\Delta >= 2$ else we would be forcing the driver to accept the second ride and the third ride would never get the chance to be matched to a driver. From Figure~\ref{fig:profit_fairness_crs_synthetic_extreme}, we can see the trend of decreasing fairness with increasing $\alpha$ (since $\beta$ controls fairness and as $\alpha$ increases $\beta$ decreases). Figure ~\ref{fig:scatter_synthetic_extreme} compares \lpalg to \gre and \uni. In our setting all profits are equal (all $w_f = 1$) hence the dots for \gre and \uni are very close by. \gre chooses the driver that is most likely to accept the trip out of all \emph{available} drivers and \uni randomly chooses a driver and then checks for availability. This is the reason behind the slight difference in competetive ratios of \uni and \gre. However, as is clear from Figure ~\ref{fig:scatter_synthetic_extreme}, our proposed algorithm, for some choice of ($\alpha, \beta$), performs much better.

	\end{document}